\documentclass{article}
\usepackage{kr}

\usepackage{times}
\usepackage{soul}
\usepackage{url}
\usepackage[hidelinks]{hyperref}
\usepackage[utf8]{inputenc}
\usepackage[small]{caption}
\usepackage{graphicx}
\usepackage{amsmath}
\usepackage{amsthm}
\usepackage{booktabs}
\usepackage{algorithm}
\usepackage{algorithmic}
\newcommand{\AR}[1]{\textcolor{black}{#1}}
\newcommand{\MVM}[1]{\textcolor{red}{MVM: #1}}
\newcommand{\MVMSugg}[1]{\textcolor{black}{#1}}

\newcommand{\todo}[1]{\textcolor{magenta}{$*$ #1}}

\newtheorem{dummytheorem}{Theorem}
\newtheorem{theorem}{Theorem}

\newtheorem{lemma}{Lemma}
\newtheorem{definition}{Definition}
\newtheorem{example}{Example}

\newcommand{\Inst}{\ensuremath{\mathcal{V}}}
\newcommand{\FForms}{\ensuremath{\mathbb{F}}}
\newcommand{\CForms}{\ensuremath{\mathbb{C}}}
\newcommand{\head}{\ensuremath{\mathit{head}}}
\newcommand{\body}{\ensuremath{\mathit{body}}}
\newcommand{\D}{\ensuremath{\mathcal{D}}}
\newcommand{\Dom}{\ensuremath{\mathcal{Dom}}}
\newcommand{\Data}{\ensuremath{\mathcal{K}_d}}
\newcommand{\I}{\ensuremath{\mathcal{I}}}
\newcommand{\Ifeat}{\ensuremath{\mathcal{I}_f}}
\newcommand{\Iclass}{\ensuremath{\mathcal{I}_c}}

\newcommand{\feat}{\ensuremath{f}}
\newcommand{\Features}{\ensuremath{\mathcal{F}}}
\newcommand{\Labels}{\ensuremath{\mathcal{C}}}
\newcommand{\model}{\ensuremath{\mathcal{M}}}
\newcommand{\inn}{\ensuremath{\mathbf{x}}}
\newcommand{\Rules}{\ensuremath{\mathcal{R}}}

\newcommand{\Exps}{\ensuremath{\mathcal{K}_e}}
\newcommand{\KB}{\ensuremath{\mathcal{K}_\model}}

\newcommand{\enforces}{\ensuremath{\sqsupseteq}}
\usepackage{xcolor}
\usepackage{amssymb}
\usepackage{MnSymbol}
\usepackage{multirow}
\usepackage{natbib}

\title{
Advancing Interactive Explainable AI via Belief Change Theory
}

\author{%
Antonio Rago$^1$\And
Maria Vanina Martinez$^2$\\
\affiliations
$^1$Department of Computing, Imperial College London, UK \\
$^2$Artificial Intelligence Research Institute (IIIA-CSIC), 
Spain\\
\emails
a.rago@imperial.ac.uk,
vmartinez@iiia.csic.es
}

\begin{document}

\maketitle

\begin{abstract}


As AI models become ever more complex and intertwined in humans' daily lives, 
greater levels of interactivity of explainable AI (XAI) methods are needed.
\MVMSugg{In this paper, we propose the use of belief change theory as a formal foundation
for operators that model the incorporation of new information\AR{, i.e. user feedback in interactive XAI,} to logical representations of data-driven classifiers.} We argue that this type of 
formalisation provides a framework and a methodology to 
develop interactive explanations in a principled 
manner, providing warranted behaviour and favouring transparency and accountability of such interactions.
Concretely, we first define a novel, logic-based formalism to 
represent 
 explanatory information shared between humans and machines. 
We then consider 
real world scenarios for interactive XAI, with different prioritisations of new and existing knowledge, where our formalism may be instantiated.
Finally, we analyse a core set of belief change postulates, discussing their suitability for our real world settings 
and pointing to particular challenges that may require the relaxation or reinterpretation of some of the theoretical assumptions underlying existing operators.

\end{abstract}

\section{Introduction}
\label{sec:intro}

To achieve the safe, regulated and trustworthy deployment of AI while maximising its potential, a number of applications benefit from interactive explanations, where 
a human provides feedback to the AI model (see \citep{Wu_22} for a recent overview).
Interactivity has also been recognised as a core tenet of ensuring \MVMSugg{that} AI is \emph{contestable} \citep{Hirsch_17,Lyons_21}, as recommended by design principles such as those of the ACM\footnote{\href{https://www.acm.org/media-center/2022/october/tpc-statement-responsible-algorithmic-systems}{https://www.acm.org/media-center/2022/october/tpc-statement-responsible-algorithmic-systems}} and enforced by legal regulations such as the GDPR\footnote{\href{https://gdpr-text.com/read/article-22/}{https://gdpr-text.com/read/article-22/}}.
Meanwhile, the field of explainable AI (XAI), with its overarching objective of fostering trust in AI models, predominantly focuses on static explanations which do not support such interactivity (see \citep{Ali_23} for an overview).
Some XAI methods provide interactivity via user feedback, e.g. in human-in-the-loop reinforcement learning \citep{Retzlaff_24}, recommender systems \citep{Rago_21
} and text classification \citep{Arous_21}, where explainability has been said to be beneficial, but this research area remains relatively unexplored.
Further, formal frameworks for interactivity \AR{in XAI} are lacking, despite their 
\AR{crucial role in trustworthiness} \citep{Ignatiev_22}, giving scant prospect for regulations on interactive XAI to be defined and systematically enforced.

In this paper, we propose the use of belief change theory~\citep{AGM85} 
within the modelling of interactive explanations \AR{for data-driven classifiers}.
We assume as given a set of explanations about a classifier, in the form of rules, e.g. as in \citep{Guidotti_18X,Ribeiro_18,Shih_18,Grover_19,Ignatiev_19
}, and we envisage the possibility of users providing feedback 
\AR{thereon}, also in the form of rules. We see {\em revision operators} as being particularly well-suited to modelling the process of feedback incorporation\AR{, as evidenced in the related setting of editing 
multi-label classifiers \citep{Coste-Marquis_21}}.

We argue that  \MVMSugg{such} formalisation\AR{s} lay 
the groundwork 
for \AR{the} design and development of interactive explanations 
that promote transparency, interpretability and  accountability in human-machine interactions. 
As an example of the importance of this topic, the recently endorsed AI Act\footnote{ \href{https://artificialintelligenceact.eu/}{https://artificialintelligenceact.eu/}} regulatory framework for the European Union, guarantees the {\em right of consumers to launch complaints and receive meaningful explanations}. Such legal requirements make it evident that novel methodologies and tools are needed to provide formal guarantees about not only 
AI models' 
behaviour but also about all related human-machine interactions. 


\AR{After covering the related literature (§\ref{sec:relatedwork}), we make the following contributions:} 

\begin{itemize}
    \item We define a novel, logic-based formalism to represent how information is shared between humans and machines, specifically classification models, in XAI (§\ref{sec:MLrep}). 
    \item We consider a set of real world scenarios of interactive XAI where our formalism may be instantiated with different prioritisations of new and existing knowledge (§\ref{sec:IntExp}).
    \item We instantiate a core set of belief revision postulates in our formalism, discussing their strengths and weaknesses (§\ref{sec:BR}), before looking ahead to what is required for belief revision to make advancements in interactive XAI (§\ref{sec:conclusions}). 
\end{itemize}


\section{Related Work}
\label{sec:relatedwork}

Within the area of belief revision the work of \cite{Falappa_02} proposes a non-prioritised revision operator based on the use
of explanations by deduction. The epistemic input is accompanied by an explanation supporting it and beliefs are dynamically qualified as defeasible or undefeasible and revised accordingly.
Recently, \cite{Coste-Marquis_21} proposed a belief change operator, called a \emph{rectification operator}, that aims to modify, according to some available background knowledge, a Boolean circuit that exhibits the same input-output behaviour as a multi-label classifier. The operation ensures that the rectified circuit complies with the background knowledge through different notions of compliance.
Though this proposal also aims to model modifications to logical representations of classifiers through belief change operators, there exist significant differences. First, we assume partial and approximate knowledge of the classifier's behaviour and therefore a potentially incomplete and not coherent logical representation of it; this has a direct impact on the analysis of suitable postulates. Second, the classifier's representation and the feedback provided by users are specified by means of rules rather than propositional logical sentences.  
We believe this encoding provides greater interpretability from a user's point of view.
Third, instead of prioritising the input or feedback, we study alternatives according to different scenarios of interactive explanations, allowing for the possibility for the logical representation to gradually differ from the original classifier specification as feedback is incorporated. 
Finally, the work from \cite{Schwind_23} proposes a series of operators
that determine how a Boolean classifier should be edited whenever it does not label a data point in the correct way. The paper studies the incorporation of positive, negative and combined (positive and negative) instances. Besides only focusing on Boolean classifiers, the differences mentioned above for the multi-label approach also hold in this case. 




\section{
Formalising Classifiers and Explanations}
\label{sec:MLrep}

In this section, we formalise classifiers' outputs and explanations based on propositional logic and formal rules, extending the language from~\citep{Amgoud_23}.\footnote{In the supplementary material 
we provide example illustrations of our approach, as well as a proof of Theorem \ref{thm:main}.} 

 We assume a single-label classification problem where $F = \{f_1,\ldots, f_m\}$ is the set of $m > 1$ features, where each $f_i \in F$ has a discrete domain $\Dom(f_i)$, and $C = \{ c_1, \ldots, c_n \}$ is the set of $n>1$ possible classes or classification labels.
We let $V = \D(f_1) \times \ldots \times \D(f_m)$ be the (combinatorial) set of all possible data points, i.e. assignments of values to all features
.
Straightforwardly, we then let a dataset be a set of data points $D \subseteq V$.
Then, a classifier $\model: D \rightarrow C$ is a total mapping\footnote{\AR{Note that $\model$ is a total mapping wrt $D$, i.e. the data points for which the classes predicted by the classifier are known.
Here, $D$ may represent any dataset, e.g. that used for training.}} 
such that for any 
$\inn \! \in \! D$, we say that $\model$ predicts class $c \!\in\! C$ iff $\model(\inn) \!=\! c$. 
For a given 
$\inn$, 
$\inn^i$ is the value $v \in \D(f_i)$ assigned to feature $f_i$.

\noindent {\bf Syntax}. To model a classifier, 
we assume
a propositional language based on two finite alphabets $\Features = f_1, \ldots, f_m$ and $\Labels = c_1, \ldots, c_n$, representing elements in $F$ and $C$, resp. 
For each symbol $f$ in $\Features$, we assume a discrete set of constants $\Dom(f)$ corresponding to the domain ($\D(f)$) of feature $f \in F$.

A \emph{feature atom} 
is of the form $(f, v)$, where $f \in \Features$ and $v \in \Dom(f)$; a feature literal is either a feature atom $a$ or $\neg a$.
On the other hand, a  \emph{classification atom} is of the form $c$, with $c \in \Labels$. 
Intuitively, a feature atom represents the fact that value $v$ is assigned to feature $f$, while a classification atom represents a set of classes (in particular, an atom represents a singleton, as we will see later). 

A feature (classification, resp.) formula is any logical formula built from feature (classification, resp.) literals using classical connectives $\neg, \wedge, \vee$.
We use $\FForms$ ($\CForms$, resp.) to denote the set of all feature (classification, resp.) formulas.

We distinguish the following set of feature formulas, intuitively to link each of them to a specific data point in $V$.

\begin{definition}
A \emph{(data)} \emph{instance} $x$ is conjunction of feature atoms such that each feature $f \in \Features$ appears exactly once.
We will call $\Inst$ the set of all possible data instances. 
\end{definition}

Intuitively, feature formulas represent sets of data points in $V$, while a data instance represents a specific data point in $V$. 
On the other hand, 
classification formulas represent sets of classification labels. The concept of a rule, defined below, allows us to map feature formulas into classification formulas, which ultimately seek to represent a mapping between data points and a set of potential classification labels.

\begin{definition}
\label{def:rules}
    A  \emph{rule} $r$ is of the form
    $\phi \Rightarrow \psi$, where $\phi \in \FForms$ and $\psi \in \CForms$. We call
    $\phi$ the body of $r$, denoted 
    $\body(r)$, and $\psi$ the head of $r$, denoted $\head(r)$. If $\phi$ is a data instance (i.e. $\phi \in \Inst$) and $c$ is a positive literal, we call $r$ an \emph{instance rule}. 
\end{definition}

Intuitively, we use a rule $r$ to
establish the set of classes, defined by the classification formula in the head of the rule,
that is assigned to the set of data instances characterised by the feature formula $\body(r)$. 
Note that when a rule establishes that a certain set of data points are assigned to a non-singleton set of classes, we interpret that any of those classes in the set could be assigned, but only one of them.

\noindent {\bf Semantics}.
Function 
$\Ifeat \!:\! \FForms \!\rightarrow\! 2^{V}$ maps feature formulas to sets of data points in $V$. Formally, for feature formulas $\phi, \psi$: \\
\noindent $\bullet$ if $\phi = (f_i,v)$, then $\Ifeat((f_i,v)) = \{\inn \in V \;|\;  \inn^i = v\}$ \\
\noindent $\bullet$ $\Ifeat(\phi \wedge \psi) = \Ifeat(\phi) \cap \Ifeat(\psi)$ \\
\noindent $\bullet$ $\Ifeat(\phi \vee \psi) = \Ifeat(\phi) \cup \Ifeat(\psi)$ \\
\noindent $\bullet$ $\Ifeat(\neg \phi) = \{\inn \in V \;|\;  \inn \nin \Ifeat(\phi)\}$  \\
\noindent $\bullet$ $\Ifeat(\bot) = \emptyset$, $\Ifeat(\top) = V$

The semantics for classification formulas is defined with $\Iclass \!:\! \CForms \!\rightarrow\! 2^{C}$, which maps classification formulas to sets of classes in $C$. Formally, for classification formulas $\phi, \psi$: \\
\noindent $\bullet$ if $\phi = c$, where $c$ is a classification atom, then $\Iclass(c) = \{c\}$ \\
\noindent $\bullet$ $\Iclass(\phi \wedge \psi) = \Iclass(\phi) \cap \Iclass(\psi)$ \\
\noindent $\bullet$ $\Iclass(\phi \vee \psi) = \Iclass(\phi) \cup \Iclass(\psi)$ \\
\noindent $\bullet$ $\Iclass(\neg \phi) = \{c \in C \;|\;  c \nin \Iclass(\phi)\}$ \\ 
\noindent $\bullet$ $\Iclass(\bot) = \emptyset$, $\Iclass(\top) = C$  

To interpret rules in this setting, we define function $\I$ such that, given a rule $r$, $\I(r) = (\Ifeat(\body(r)), \Iclass(\head(r)))$. Here, $\I$ maps
the set of  data points represented by the formula $\body(r)$ into a set of classes determined by $\head(r)$.

Using the language defined above, we can logically model classifiers, and explanations therefor (defined later),
by means of
rules since they express mappings of sets of data points into sets of classification labels. 
In addition to extending the representation language from~\citep{Amgoud_23}, note that the spirit of the aforementioned paper is different to ours. In that work, the authors formally define functions that generate different types of explanations and study their properties in relation to existence and correctness. In this work, we assume explanations from a classifier have already been provided in the form of rules and we model the interactions with the model's users through operations that could update such rules as a result of the user's feedback.


Based on the semantics, we can now  define notions that help us establish relationships among 
rules. The first is {\em enforcement}: intuitively, a set of rules enforces another set of rules whenever  
every possible assignment of data points to a class that the 
\AR{enforced} set of rules represents is also an assignment that is represented by the 
\AR{enforcing set}. 

\begin{definition}

    Given sets of rules $\Rules_i$ and $\Rules_j$, $\Rules_i$ {\em enforces} $\Rules_j$, denoted $\Rules_i \enforces \Rules_j$, iff $\forall r_j \in \Rules_j$, $\forall \inn \in \Ifeat(\body(r_j))$, $\exists r_i \in \Rules_i$, such that $\inn \in \Ifeat(\body(r_i))$ and $\Iclass(\head(r_i)) \subseteq \Iclass(\head(r_j))$.

\end{definition}


The second notion we define is {\em consistency}, requiring that sets of rules do not assign incompatible labels to data points. 

\begin{definition}
\label{def:consistency}
    Given a set of rules $\Rules$,  $\Rules$ is {\em consistent} iff $\forall \inn \in \bigcup_{r \in \Rules} \Ifeat(\body(r))$, $\forall r_i,r_j \in \Rules$ such that $\inn \in \Ifeat(\body(r_i))$ and $\inn \in \Ifeat(\body(r_j))$ then
    $\Iclass(\head(r_i)) \cap \Iclass(\head(r_j)) \neq \emptyset$. Otherwise, $\Rules$ is {\em inconsistent}.
\end{definition}


The notion of \emph{coherence} defined below aims to capture the relationship between rules and models. Intuitively, given a threshold $\tau \in [0,1]$, a set of rules is $\tau$-coherent with a model iff the proportion of instances captured by the body of every rule, such that the model's classification of the instance is included in the head of the rule, is at least $\tau$. \AR{This generalises the notion of \emph{compatibility} in~\citep{Amgoud_23} allowing a percentage of the classifications described by the set of rules to differ from the classifications provided by the classifier.}


\begin{definition}
\label{def:coherence}
    Given a classifier $\model$ and a \emph{threshold} $\tau \in [0,1]$, we say that rule $r$ is \emph{$\tau$-coherent} with $\model$ iff $\Ifeat(body(r)) \cap D = \emptyset$ or:
    $$\frac{|\{ \inn \in \Ifeat(body(r)) \cap D | \model(\inn) \in \Iclass(head(r)) \}|}{| \Ifeat(body(r)) \cap D |} \geq \tau$$
    We say that a set of rules $\Rules$ is \emph{$\tau$-coherent} with $\model$ iff $\forall r \in \Rules$, $r$ is \emph{$\tau$-coherent}. Whenever $\tau = 1$, we drop the $\tau$ prefix and say that a (set of) rule(s) is \emph{coherent} with $\model$.
\end{definition}

\begin{lemma}
\label{lemma:coherence}
    Given a classifier $\model$, a rule $r$ is coherent with $\model$ iff $\forall \inn \in \Ifeat(body(r)) \cap D$, $\model(\inn) \in \Iclass(head(r))$.
\end{lemma}


Next, we formalise whether a set of rules completely (and exclusively) represents the set of known data points $D$. 

\begin{definition}
\label{def:completeness}
    Given a classifier $\model$, a set of rules $\Rules$ is \emph{complete} for $\model$ iff $\Rules$ is a  set of instance rules such that $|\Rules| = |D|$ and $\forall \inn \in D$, $\exists r \in \Rules$ where $\I(r) = (\{\inn\}, \{ \model(\inn) \})$. 
\end{definition}

We now represent the knowledge we have about a classifier by means of rules as follows. 

\begin{definition}
\label{def:expkb}
    Given a classifier $\model$, an \emph{explanation knowledge base} for $\model$ is a set $\KB = \Data \cup \Exps$, where $\Data$ is a set of instance rules, called the \emph{data}, and $\Exps$ is a set of general rules, called the \emph{explanations}.
\end{definition}

Our intention is for $\Data$ to represent data points for which the classification is known, these may come either from training or evaluation phases or from previous use of the classifier. $\Data$ logically represents the classifier,
such that they encode exactly the same classifications. In addition to this, 
there exist different methods in the literature to elicit behavioural patterns from classifiers, often expressed as rules functioning as explanations, we use $\Exps$ to represent that kind of knowledge. Although consistency is generally expected, a priori we impose no restrictions of coherence of $\Exps$ with $\model$, as they represent tentative knowledge obtained from potentially imprecise methods. These explanation rules may have been extracted by existing formal methods for explaining (discrete) classifiers from the literature, such as~\citep{Guidotti_18X,Ribeiro_18,Shih_18,Grover_19,Ignatiev_19
}.
Moreover, $\tau$-coherence corresponds to the notion of precision in \citep{Ribeiro_18}, and could be used to allow for tolerance in
the correctness of explanations.
However, for this paper we will assume that $\tau=1$.



\begin{theorem}
\label{thm:main}
    Given a classifier $\model$ and an explanation knowledge base $\KB =  \Data \cup \Exps $, where $\Data$ is complete for and coherent with $\model$, a set of rules $\Rules$ is coherent with $\model$ iff $\Rules \cup \Data$ is consistent. 
\end{theorem}

This result shows that preserving consistency with $\Data$ preserves coherence with $\model$ (see the consistency postulate in §\ref{sec:BR}) whenever $\Data$ is complete and coherent with $\model$. Note, however, 
that we do not make this assumption in general, since, as discussed later, we aim for a framework that is tolerant to inconsistency and in which the logical representation of $\model$ may 
differ from $\model$ due to feedback incorporation.



\section{Interactive Explanations}\label{sec:IntExp}

We now demonstrate how interactive explanations may be modelled with our formalism, and consider how such explanations may be deployed in real world settings.
We consider interactive explanations which give users the ability to provide feedback to classifiers in a number of ways, in the form of rules (Definition~\ref{def:rules}), which we call here \emph{feedback}. 
When a rule is provided as feedback (we limit to single rules), the goal is to analyse how/if this knowledge can be incorporated, possibly modifying both the explanation knowledge base and the feedback itself. 
\MVMSugg{This type of feedback mirrors rule-based explanations from XAI (as we discuss in §\ref{sec:MLrep}) that intuitively represent knowledge in any domain and easily translate to \AR{and} from natural language.}

We define the following basic desiderata for this process: 

\begin{itemize}
    \item \emph{Constrained Inconsistency}: specific scenarios 
    may require some tolerance to inconsistency, 
    e.g. requiring only $\Data$ to be kept consistent after an interaction.
    \item \emph{Bounded model incoherence}: while we expect $\Data$ to be coherent with $\model$, the weaker notion of $\tau$-coherence of the explanations $\Exps$ with $\model$ could be accepted, for a given $\tau$.
    \item \emph{Minimal information loss}: the information contained in $
    \KB \cup \{ r \}$ should be modified or removed minimally, and only when it jeopardises the above desiderata. 
\end{itemize}

Belief revision incorporates new information following two main principles: consistency (preservation) and minimal change. Our desiderata for interactive explanations coincide with these aims in minimising the amount of information loss. However, we relax the notion of consistency, and allow the classifier and its logical representation to drift in a restricted manner through the notion of $\tau$-coherence. 

We now give a (non-exhaustive) set of real world application settings where interactive explanations may be deployed. 
We base our settings on those proposed in \citep{Retzlaff_24} for human-in-the-loop systems.

In the first setting, named \textbf{S1}, we envisage a classification model which is in \textbf{development}, e.g. being debugged by a developer as in \citep{Lertvittayakumjorn_20}. Here, the user provides feedback to update, and correct, the model. In this case, the model's trust in the feedback can be regarded as \textbf{credulous}, since the model should be updated to align with the feedback, i.e. any feedback $r$ takes priority over the existing knowledge (informally represented with $r \succ \KB$). For example, if a user provides $r$ which contradicts an instance rule representing an existing data point in $\Data$, e.g. due to the default settings of the model or changing preferences of the user, we would like to incorporate $r$  and update the conflicting instance rule to align with the new conditions specified by the user.

In the second setting, \textbf{S2}, we introduce a model which is being refined in an \textbf{evaluation} stage by group of users, e.g. as in domain expert information fusion \citep{Holzinger_21}, before the model is deployed at scale. 
In this case, a single model is being updated by feedback from multiple users, and so the model's trust in the feedback must be \textbf{balanced} with that in the existing knowledge. Here, a single user's feedback should not necessarily take precedence over existing knowledge (informally, $r \simeq \KB$), and so both the new and the existing knowledge may be modified in order for consistency and coherence with the model to be maintained with minimal information loss. 
For example, if a user provides $r$ which contradicts $\Exps$, it may be desirable to preserve $\Data$ but modify $\Exps$ or $r$ by weakening or rejection to incorporate as much of the new knowledge as possible.

In the final setting, \textbf{S3}, we consider a model which has already undergone commercial \textbf{deployment} at scale, but allows for feedback from the sizeable group of users for completing gaps in the knowledge, e.g. as in autonomous vehicles \citep{Wu_23}. 
Here, the model will be updated by users' feedback, but due to the size of this group and the fact that the model has already been deployed commercially, e.g. potentially raising legal issues, the trust in the feedback is \textbf{sceptical}, and it thus prioritises existing over new knowledge (informally, $r \prec \KB$). The new knowledge can thus be modified in order to ensure its consistency with the existing knowledge. 
For example, if the user provides some $r$ which does not violate the consistency of $\Exps$ or the coherence of $\Data$, then it may be incorporated to $\KB$ as is to minimise information loss.
Meanwhile, if it contradicts $
\KB$, then we may preserve $
\KB$ while only part of $r$ may be incorporated.

It is important to note that these modifications are not performed over the model itself but its logical representation, creating in each interaction a new knowledge base that may differ substantially from the original knowledge base (and the model).
A distance between different versions of the knowledge base could be measured through $\tau$-coherence or more conventional distance measures, and be used as a way of checking the effect of feedback, e.g. as an indicator for when the retraining of the model is required. 


Having presented out motivational scenarios, in the next section we analyse the suitability of belief revision operators to model interactions with explanation knowledge bases.

\section{Revision of Explanation KBs}
\label{sec:BR}

One of the main contributions of the foundational models of belief change is the development of a style of research and development methodology based on providing axiomatic characterisations of the operators' behaviour in terms of  postulates.
The postulates focus on 
conditioning and constraining the inputs and the results of the operators, rather than providing insights into how the results are achieved. Representation theorems are used both to provide semantic characterisations for belief change operators, as well as linking these characterisations to computational implementations, providing provable guarantees on the behaviour of such algorithms. 
In the following we analyse a core set of postulates for belief base revision~\citep{Hansson93}, translate them in our logical setting and discuss their suitability with respect to the different scenarios of interactive explanations.

In this work, we adopt the approach to belief revision known as \emph{base revision}, where existing knowledge is represented as a finite set of formulas~\citep{Hansson93}, which we call an explanation knowledge base $\KB$, as described in §\ref{sec:MLrep}. The new information consists of a single rule $r$ that is obtained from the interaction with 
the user(s) of the model.
In the following analysis we use $\KB * r$ to describe the application (and the results) of operator $*$ over the existing knowledge base $\KB$ and input (feedback) $r$.





\textbf{Success} 
states that the epistemic input is always accepted, i.e. new knowledge is prioritised.
This can be formalised in our framework by means of 
 our notion of enforcement of the feedback rule, i.e. $\KB * r \enforces \{ r \}$.
In setting S1, the success postulate can be used to enforce the feedback taking priority over the existing $\KB$ (in the presence of inconsistency). Prioritised revision operators are suitable for this setting, while this is not the case for (possibly S2 and) S3, where the existing knowledge should be prioritised. A first approach to define non-prioritised behaviour could be modelled by a simple {\em relative success} postulate~\citep{CLBasesjlc03}, which states that either the input is fully (explicitly) accepted or rejected, i.e. either $r \in \KB * r$ or $\KB * r = \KB$, resp. 
More fine-grained alternatives would allow for the specification of conditions 
under which the input could be fully or partially accepted.
For instance, {\em weak success}~\citep{selectiveRevBaseskr20} 
may state that if $\KB \! \cup \! \{ r \}$ is consistent then $\KB * r \enforces \{ r \}$. Meanwhile, \emph{proxy success} and \emph{weak proxy success} \citep{selectiveRevBaseskr20} state that the revision should incorporate part of the input, e.g. to ensure all users' feedback plays a part in S2. Formally, proxy success could be defined requiring that $\exists r'$ such that $\{ r \} \enforces \{ r' \}$, $\KB * r \enforces \{ r' \}$ and $\KB * r = \KB * r'$. In weak proxy success, $r'$ is not conditioned by $r$: $\exists r'$ such that 
$\KB * r \enforces \{ r' \}$ and $\KB * r = \KB * r'$.
These weakened postulates seem appropriate for S2 and S3, where gaps in $\KB$ could be filled more often with these weaker constraints, but less so for S1, where success may be preferred given the trust in the user here.
However, any version of success that allows for the incorporation of only part of a rule could induce bias in the dataset. A potentially problematic example could be when only a stricter version of a feedback rule, covering only a subset of a feature (e.g. an ethnic minority in a population), rather than its entirety, is incorporated to $\KB$. 



\textbf{Inclusion} states that the only addition to the existing knowledge can be the feedback itself, instantiated in our setting as $\KB * r \subseteq \KB \cup \{r\}$. 
This raises issues in our settings, since it may be desirable that $\Data$, $\Exps$ or both are modified, for instance making rules more specific. 
In S1, it is desirable that we incorporate $r$ as is, but we may wish for $\Exps$ to be adapted to this new information. Also, in S2 and S3, we may want to incorporate only part of $r$, since it may be unrealistic to incorporate $r$ in its entirety given the higher priority of $\KB$.
An alternative is \emph{weak inclusion}~\citep{selectiveRevBaseskr20}, which states that if $r \in \KB * r$, then $\KB * r \subseteq \KB \cup \{r\}$. This relaxation alleviates the second aforementioned issue, and we would thus posit that this is desirable in S3, where existing and new information is restricted from modification, e.g. from a legal standpoint if users have already seen it. 
However, in S2 we would expect that $\Exps$ being adapted to $r$ would be more suitable.
We thus propose three alternate formulations of inclusion based on our notion of enforcement, prioritising the suitable data in each setting.
For S1, we suggest that $\KB * r \subseteq A \cup \{r\}$, where $\KB \enforces A$, allowing the existing explanations to adapt to the new information.
For S2, we suggest that $\KB * r \subseteq A$, where $
\KB \cup \{r\} \enforces A$, allowing for the modification of both existing and new information.
For S3, we suggest that
$\KB * r \subseteq 
\KB \cup A$, where $\{r\} \enforces A$, ensuring that only the feedback is modified.

\textbf{Consistency} 
conventionally requires that a knowledge base becomes consistent after the revision, even if it is not so beforehand.
Formally, $\KB * r$ is required to be consistent,\footnote{A singleton set containing $r$ is  consistent by Definition~\ref{def:consistency}, so our version of the postulate does not condition on the consistency of the input.
Allowing for sets of feedback rules, as in \emph{multiple revision}~\citep{FuhrmannH94}, is future work.} which, by Theorem~\ref{thm:main}, may cover the first two of our desiderata whenever $\Data$ is coherent with $\model$.
Thus, the notion of consistency seems to be desirable across our settings, whenever neither consistency nor coherence is relaxed. In particular, \emph{consistency preservation}~\citep{AGM85}, which requires that a consistent KB be consistent after operating (adding the condition that 
$\KB$ is consistent to the consistency postulate above) seems suitable for all settings, since it requires feedback not introduce such inconsistencies, rather than requiring it fix any which already exist.
Note, however, that it may be the case that we are interested in only $\Data$ remaining/becoming consistent after the revision, given the tentative and approximate nature of $\Exps$. An alternative to be considered is to ensure that the revision does not increase the amount of inconsistency (given a measure for it~\citep{Thimm16,grant2018measuring}) in $\KB$ or $\Data$.

\textbf{Relevance} 
concerns minimal change of existing knowledge, stating that if $r' \in \KB$ and $r' \nin \KB * r$, then there is a set of rules $\Rules$ such that $\KB * r \subseteq \Rules \subseteq 
\KB \cup \{ r \}$, $\Rules$ is consistent and $\Rules \cup \{r'\}$ is inconsistent.
Relevance formalises our third desideratum in terms of only removing  information from the data or explanations if it were inconsistent with the feedback being provided by the user, rendering it suitable across our settings.
This postulate has important implications for data protection, ensuring that the non-conflicting knowledge is preserved and therefore is desirable in all three settings. 
However, as defined above, this postulate forces $\KB * r$ to be a subset of $\KB \cup \{r\}$; in light of our previous discussion, if we want to have the possibility of not only deleting but modifying both the existing knowledge and feedback, we could consider a weaker notion closer to the postulate known as \emph{core-retainment}: in our setting this could be formalised as
 if $r' \in 
 \KB$ and $r' \nin \KB * r$, then there is a set of rules $\Rules$ such that $ \Rules \subseteq 
 \KB \cup \{ r \}$, $\Rules$ is consistent but $\Rules \cup \{r'\}$ is inconsistent.
 

\textbf{Uniformity}, formulated in our setting, states that if $\forall \Rules \subseteq 
\KB$, $\Rules \cup \{r\}$ is  inconsistent if and only if $\Rules \cup \{r'\}$ is  inconsistent, then $
\KB \cap (\KB * r) = 
\KB \cap (\KB * r')$. The intuition here is that if $r$ and $r'$ are inconsistent with $\KB$ in the exact same way, revising by either retains the same knowledge from $\KB$. 
Once again, uniformity seems to be appropriate across the settings, guaranteeing the regularity of the effects of feedback, which could be useful for ensuring that regulatory guidelines are met.

\section{Discussion and Future Work}
\label{sec:conclusions}

\begin{table}[t]
\centering
\begin{tabular}{lccc}
\cline{2-4}
 &
\textbf{S1} &
\textbf{S2} &
\textbf{S3} \\ 
\hline
Trust & 
\!Credulous\! & 
\!Balanced\! &
\!Sceptical\! \\ 
Setting & 
\!Development\! & 
\!Evaluation\! &
\!Deployment\! \\ 
Users & 
\!Single\! & 
\!Small-Scale\! &
\!Large-Scale\! \\ 
Priority & 
$r \!\succ\! \KB$ & 
$r \!\simeq\! \KB$ &
$r \!\prec\! \KB$ \\ 
\hline
Success & 
$\checkmark$ & 
\!$\checkmark^{rs,ws,ps,wps}$\! &
\!$\checkmark^{ws,ps,wps}$\! \\ 
Inclusion & 
- & 
- &
$\checkmark^{wi}$ \\ 
Consistency & 
$\checkmark^{cp}$ & 
$\checkmark^{cp}$ &
$\checkmark^{cp}$ \\
Relevance & 
$\checkmark$ & 
$\checkmark$ &
$\checkmark$ \\ 
Uniformity & 
$\checkmark$ & 
$\checkmark$ &
$\checkmark$ \\ 
\hline
\end{tabular}
\caption{Characteristics of our real world settings and assessment of postulates, where $\checkmark$ indicates a postulate is desirable, $\checkmark^x$ indicates that only a weaker postulate is desirable and $-$ indicates novel postulates may be required, with $x$ indicating the following weaker postulates: 
\underline{r}elative \underline{s}uccess, \underline{w}eak \underline{s}uccess, 
\underline{p}roxy \underline{s}uccess,
\underline{w}eak \underline{p}roxy \underline{s}uccess, 
\underline{w}eak \underline{i}nclusion and
\underline{c}onsistency \underline{p}reservation. 
}
\label{table:results}
\end{table}

 Table \ref{table:results} summarises the results of our analysis. Some of the existing postulates are suitable for all of these settings in their original form, i.e. relevance and uniformity, while the others require alternate versions from the literature. However, across all studied postulates, we believe that there is scope for novel, tailored versions which may be more suitable in the individual settings, as we have indicated. Indeed, even in the cases where there are suitable postulates, others may be preferable, e.g. as we suggested for success. We believe that this highlights many fruitful avenues for future work. Among these, a next step is to characterise the behaviour of each setting with a specific set of postulates and provide the corresponding constructions.
Regarding constructions, it seems possible to implement S1 with minimal modifications to traditional belief revision base operators such as \emph{partial meet} and \emph{kernel}~\citep{Hansson93}.
The other two of our envisaged settings lend themselves to non-prioritised revisions that could be implemented through operators such as \emph{credibility limited}~\citep{CLBasesjlc03} and \emph{screened revision}~\citep{Makinson1997-MAKSR},  in which a portion of the knowledge $\mathcal{K}_p \subseteq \Data \cup \Exps$ 
is protected from the revision. 
For example, it may be the case that unless data points from the dataset $\Data$ are explicitly mentioned in the feedback, then we protect $\Data$ from changes, i.e. $\mathcal{K}_p = \Data \setminus \{ r \}$, and modify only explanations. 
In S2, $\Exps$ may be seen as being modifiable while $\Data$ is protected (no matter what $r$ is being provided), i.e. $\mathcal{K}_p = \Data$, for example if the dataset has been curated to be unbiased.
Another case could be when a subset of $\Data \cup \Exps$ needs to be protected from the revision, for example rules representing data points or explanations which have already been delivered to users, \emph{semi-revision}~\citep{semirevision97} could be useful here as it allows $r$ to be discarded. 
Our analysis also suggests that for S2 and S3 it may be desirable to only retain part of the information contained in $r$. The closest operator in the literature that behaves in this way is \emph{selective revision}~\citep{selectiveRevBaseskr20}. 
All these operators are implemented based on classical AGM  operators, either checking conditions or modifying the input before applying an AGM revision operator or recurring to other operators such us consolidation (restoring consistency) over $\KB \cup \{r\}$. For setting S2 and S3 we may need to combine their implementations.

In light of the discussion about consistency, we need to define alternative postulates that better satisfy our proposed desiderata, including tolerating some degree of inconsistency 
and $\tau$-coherence of $\Exps$ with the model for $\tau \neq 1$. 
Finally, our analysis assumes 
independence of interactions
and that feedback consists of a single rule. 
Operators such as those based on \emph{iterative revision}~\citep{DarwicheP94} and \emph{improvement}~\citep{KoniecznyP08} are worth studying for continuous 
feedback, e.g. coming from different users or over time, as well as multiple revision~\citep{FuhrmannH94} in order to allow arbitrary sets of rules as feedback. 
We leave exploration of these lines of research to future work.


\newpage

\quad

\newpage

\section*{Acknowledgments}

Rago was partially funded by the ERC under the
EU’s Horizon 2020 research and innovation programme (
No. 101020934, ADIX) and by J.P. Morgan and by the Royal
Academy of Engineering, UK. 
Martinez was partially supported by the Spanish project PID2022-139835NB-C21 funded by MCIN/AEI/10.13039/501100011033, PIE 20235AT010 and iTrust (PCI2022-135010-2).
The authors thank Musaab Ahmed Mahjoub Ahmed for his feedback.


\bibliographystyle{kr}
\bibliography{bib_lean}

\newpage

\section*{Supplementary Material}

First, we prove the theoretical result in the paper.

\begin{dummytheorem}
    Given a classifier $\model$ and an explanation knowledge base $\KB =  \Data \cup \Exps $, where $\Data$ is complete for and coherent with $\model$, a set of rules $\Rules$ is coherent with $\model$ iff $\Rules \cup \Data$ is consistent. 
\end{dummytheorem}

\begin{proof}
    %
    First we prove by contradiction that if $\Rules$ is coherent with $\model$, then $\Rules \cup \Data$ is consistent. 
    We therefore let $\Rules \cup \Data$ be inconsistent and thus, by Definition \ref{def:consistency}, $\exists r_i, r_j, \in \Rules \cup \Data$ such that $\Ifeat(body(r_i)) \cap \Ifeat(body(r_j)) \neq \emptyset$ and $\Iclass(head(r_i)) \cap \Iclass(head(r_j)) = \emptyset$.
    First, we know that, like $\Rules$, $\Data$ is consistent by Definition \ref{def:completeness}, and thus the inconsistency must lie in $\Rules \cup \Data$ only.
    By Definition \ref{def:completeness}, we can see that $\forall r_k \in \Data$, $\I(r_k) = (\{\inn_k\}, \{ \model(\inn_k) \})$ where $\inn_k \in D$. 
    However, by Definition \ref{def:coherence} and Lemma \ref{lemma:coherence}, we can see that $\forall r_l \in \Rules$ such that $\exists \inn_l \in \Ifeat(body(r_l)) \cap D$, $\model(\inn_l) \in \Iclass(head(r_l))$.
    Thus, $\forall \inn_i \in \Ifeat(body(r_i)) \cap \Ifeat(body(r_j))$, $\exists \model(\inn_i) \in \Iclass(head(r_i)) \cap \Iclass(head(r_j))$ and so we have the contradiction. \\
    We now we prove, again by contradiction, that if $\Rules \cup \Data$ is consistent, then $\Rules$ is coherent with $\model$. 
    We therefore let $\Rules$ be not coherent with $\model$ and thus, by Definition \ref{def:coherence} and Lemma \ref{lemma:coherence}, $\exists r_i \in \Rules$ such that $\exists \inn_i \in \Ifeat(body(r_i)) \cap D$ where $\model(\inn_i) \nin \Iclass(head(r_i))$.
    By Definition \ref{def:completeness}, we can see that $\forall \inn_j \in D$, $\exists r_j \in \Data$ where $\I(r_j) = (\{\inn_j\}, \{ \model(\inn_j) \})$.
    Thus, since $\inn_i \in \Ifeat(body(r_i)) \cap D$, $\exists r_k \in \Data$ such that $\Ifeat(body(r_k)) = \{ \inn_i \}$ and $\Iclass(head(r_k)) = \{ \model(\inn_i) \}$.
    However, by Definition \ref{def:consistency}, we can see that since $\Ifeat(body(r_i)) \cap \Ifeat(body(r_k)) \neq \emptyset$, then $\Iclass(head(r_i)) \cap \Iclass(head(r_k)) \neq \emptyset$.  
    Thus, $\model(\inn_i) \in \Iclass(head(r_i)) \cap \Iclass(head(r_k))$, and so we have the contradiction.
\end{proof}

Next we give examples demonstrating each component of our methodology. First, we demonstrate the classification problem we consider.

\begin{example}
\label{ex:problem}
Let a simple classification problem consist of a dataset 
$D = \{ \inn_1, \inn_2, \inn_3, \inn_4, \inn_5, \inn_6\}$, for which the set of (binary) features is $F=\{\feat_1, \feat_2,\feat_3\}$ and the set of classes is $C = \{c_1,c_2,c_3\}$, and a classifier $\model$, the classifications for which are given in Table~\ref{table:example}. 
For example, data point $\inn_1$ is such that 
$\inn_1^1 = 1$, $\inn_1^2 = 1$ and $\inn_1^3 = 0$, i.e. features $f_1$, $f_2$ and $f_3$ are assigned the values 1, 1 and 0, resp.,
and $\model(\inn_1) = c_1$, i.e. $\model$ predicts class $c_1$  
for this data point.
Note that for other (unknown) data points, e.g. that where the three feature values are set to $0$, 
the class predicted by $\model$ is undefined.

\begin{table}[ht]
    \begin{center}
    \begin{tabular}{ccccc}
    \cline{2-5}
     &
    $\feat_1$ &
    $\feat_2$ &
    $\feat_3$ &
    $\model(\inn_i)$ \\
    \hline
    %
    $\inn_1$ &
    $1$ &
    $1$ &
    $0$ &
    $c_1$ \\
    %
    $\inn_2$ &
    $0$ &
    $1$ &
    $1$ &
    $c_2$ \\
    %
    $\inn_3$ &
    $1$ &
    $0$ &
    $1$ &
    $c_2$ \\
    %
    $\inn_4$ &
    $1$ &
    $1$ &
    $1$ &
    $c_1$ \\
    %
    $\inn_5$ &
    $1$ &
    $0$ &
    $0$ &
    $c_3$ \\
    %
    $\inn_6$ &
    $0$ &
    $1$ &
    $0$ &
    $c_1$ \\
    \hline
    %
    %
    \end{tabular}
    \end{center}
    \protect\caption{A simple dataset with the classes predicted by the classifier $\model$ described in Example \ref{ex:problem}.} 
    \label{table:example}
\end{table}
\end{example}

Next, we illustrate the rules representing the classification problem. 

\begin{example}
\label{ex:rules}
For the classification problem described in Example~\ref{ex:problem}, we can represent the data points in $D$ as the following instance rules:
\begin{align}
    r_1 \!:& \quad  (f_1, 1) \wedge (f_2, 1) \wedge (f_3, 0) \Rightarrow c_1 \nonumber \\
    r_2 \!:& \quad  (f_1, 0) \wedge (f_2, 1) \wedge (f_3, 1) \Rightarrow c_2 \nonumber \\
    r_3 \!:& \quad  (f_1, 1) \wedge (f_2, 0) \wedge (f_3, 1) \Rightarrow c_2 \nonumber \\
    r_4 \!:& \quad  (f_1, 1) \wedge (f_2, 1) \wedge (f_3, 1) \Rightarrow c_1 \nonumber \\
    r_5 \!:& \quad  (f_1, 1) \wedge (f_2, 0) \wedge (f_3, 0) \Rightarrow c_3 \nonumber \\
    r_6 \!:& \quad  (f_1, 0) \wedge (f_2, 1) \wedge (f_3, 0) \Rightarrow c_1 \nonumber    
\end{align}
Here, $r_1$ says that the data point with value $1$ for $f_1$ and $f_2$, and value $0$ for $f_3$ is classified as $c_1$.
We can also represent more general rules, such as the following, which assigns class $c_1$ to all data points that have a value of $1$ for both $f_1$ and $f_2$ are classified as $c_1$:
\begin{align}
    r_x \!: \quad  (f_1, 1) \wedge (f_2, 1) \Rightarrow c_1 \nonumber
\end{align}
The following example of a more general rule makes use of the negative literal in the head of the rule, assigning either class $c_1$ or (exclusively) class $c_2$ to the data points that 
have either a value of $1$ for $f_2$ or a value of $1$ for $f_3$:
\begin{align}
    r_y \!: \quad  (f_2, 1) \vee (f_3, 1) \Rightarrow \neg c_3 \nonumber
\end{align}
\end{example}

We then exemplify the semantics of the rules. 

\begin{example}
\label{ex:semantics}
It can be seen that $\Ifeat((f_1,1)) = \{ \inn_1, \inn_3, \inn_4, \inn_5 \}$, $\Ifeat((f_3,0)) = \{ \inn_1, \inn_5, \inn_6 \}$ and $\Iclass(\neg{c_1}) = \{ c_2,c_3 \}$. 
    Further, $\Ifeat((f_1,1) \wedge (f_3,0))= \{ \inn_1, \inn_5 \}$, $\Ifeat((f_1,1) \vee (f_3,0))= \{ \inn_1, \inn_3, \inn_4, \inn_5, \inn_6 \}$ and $\Ifeat(\neg(f_1,1)) = \{ \inn_2, \inn_6 \}$.
    Meanwhile, for the instance rules defined in Example \ref{ex:rules}, we have that for $1 \leq i \leq 6$, $\I(r_i) = (\{ \inn_i \}, \{\model(\inn_i) \})$, e.g. $\I(r_1) = (\{ \inn_1 \}, \{ c_1 \})$. 
    Similarly, for the more general rules, we have $\I(r_x) = (\{ \inn_1, \inn_4 \}, \{ c_1 \})$ and $\I(r_y) = (\{ \inn_1, \inn_2, \inn_3 , \inn_4, \inn_6 \}, \{ c_1, c_2 \})$.
\end{example}

The next example illustrates the notion of enforcement between sets of rules.

\begin{example}

For a simple example of enforcement, consider the sets of rules $\Rules_i = \{ r_x \}$ and $\Rules_j = \{ r_1 \}$, using the rules from Example~\ref{ex:rules}. 
We will show that $\Rules_i$ enforces $\Rules_j$, i.e. $\Rules_i \enforces \Rules_j$.
This is because for each rule in $\Rules_j$, namely $r_1$, $\forall \inn \in \Ifeat(\body(r_1)) = \{\inn_1 \}$,
there exists a rule in $\Rules_i$, namely $r_x$, such that $\inn \in \Ifeat(\body(r_x)) = \{\inn_1, \inn_4 \}$ and $\Iclass(\head(r_x)) =\{c_1\} \subseteq \Iclass(\head(r_1)) = \{c_1\}$.


Now let us introduce $\Rules_k =  \{ r_1, r_y \}$. It can be shown that $\Rules_i$ does not enforce $\Rules_k$, i.e. $\Rules_i \not{\enforces} \Rules_k$. We know that the requirements for enforcement hold for $r_1$ from the previous example, so let us consider $r_y$. Enforcement would require that $\forall \inn \in \Ifeat(\body(r_y))$,
there exists a rule in $\Rules_i$, namely $r_x$, such that $\inn \in \Ifeat(\body(r_x)) = \{\inn_1, \inn_4 \}$ and $\Iclass(\head(r_x)) =\{c_1\} \subseteq \Iclass(\head(r_y)) = \{c_1, c_2\}$. We can see that the second condition holds here but $\inn_2, \inn_3 , \inn_6 \in \Ifeat(\body(r_y))$ and $\inn_2, \inn_3 , \inn_6 \nin \Ifeat(\body(r_x))$ so the first condition does not hold.

In a similar manner, we can see that if we let $\Rules_l =  \{ r_2, r_3, r_6, r_x \}$, then $\Rules_l \enforces \Rules_k$.

\end{example}

Next to be exemplified is our concept of consistency.

\begin{example}
\label{ex:consistency}
    If we consider the rules from Example \ref{ex:rules}, it can be easily seen that the set $\{ r_1, r_2, r_3, r_4, r_5, r_6, r_x, r_y \}$ is consistent.
    For example, consider that, 
    for $r_x$, $\Ifeat(body(r_x)) = \{ \inn_1, \inn_4 \}$ and $\Iclass(head(r_x)) = \{ c_1 \}$, and 
    for $r_y$, $\Ifeat(body(r_y)) = \{ \inn_1, \inn_2, \inn_3, \inn_4, \inn_6 \}$ and $\Iclass(head(r_y)) = \{ c_1, c_2\}$. Here, $\Ifeat(body(r_x)) \cap \Ifeat(body(r_y)) = \{ \inn_1, \inn_4 \} \neq \emptyset$, but since $\Iclass(head(r_x)) \cap \Iclass(head(r_y)) = \{ c_1 \} \neq \emptyset$, $\{r_x,r_y\}$ is consistent. 
    Meanwhile, for 
    $r_2$, $\Ifeat(body(r_2)) = \{ \inn_2 \}$ and $\Iclass(head(r_2)) = \{ c_2 \}$, so $body(r_x) \cap body(r_2) = \emptyset$, and so $\{r_x, r_2\}$ is consistent, and so on.
However, let us introduce another rule: 
    \begin{align}
        r_z \!: \quad  (f_1, 1) \Rightarrow c_1 \nonumber
    \end{align}
where 
$\Ifeat(body(r_z)) = \{ \inn_1, \inn_3, \inn_4, \inn_5 \}$ and $\Iclass(head(r_z)) = \{ c_1 \}$. It can be seen that $\{r_z, r_3\}$ is inconsistent, where 
$\Ifeat(body(r_3)) = \{ \inn_3 \}$ and $\Iclass(head(r_3)) = \{ c_2 \}$, since $\Ifeat(body(r_z)) \cap \Ifeat(body(r_3)) = \{ \inn_3 \} \neq \emptyset$, but $\Iclass(head(r_z)) \cap \Iclass(head(r_3)) = \emptyset$. 
    Similarly, $\{r_z, r_5\}$ is inconsistent, where 
    $\Ifeat(body(r_5)) = \{ \inn_5 \}$ and $\Iclass(head(r_5)) = \{ c_3 \}$, since $\Ifeat(body(r_z)) \cap \Ifeat(body(r_5)) = \{ \inn_5 \} \neq \emptyset$, but $\Iclass(head(r_z)) \cap \Iclass(head(r_5)) = \emptyset$.
\end{example}

Our next example illustrates our notion of model coherence.

\begin{example}
\label{ex:kb}
The problem from Example~\ref{ex:problem} results in an explanation knowledge base for $\model$, $\KB = \Data \cup \Exps$, where $\Data = \{ r_1, r_2, r_3, r_4, r_5, r_6 \}$. Meanwhile, for $\Exps$ it could be the case that $\Exps = \Data$, $\Exps = \{ r_x, r_y, r_z \}$ or $\Exps = \Data \cup \{ r_x, r_y, r_z \}$, for example.
    Regarding the coherence of the rules, for $r_i$, where $1 \leq i \leq 6$, i.e. for all the instance rules, $\Ifeat(body(r_i)) = \{ \inn_i \}$ and $\model(\inn_i) \in \Iclass(head(r_i))$, and so $r_i$ is coherent with $\model$. 
    For $r_x$, we have $\Ifeat(body(r_x)) = \{ \inn_1, \inn_4 \}$, $\model(\inn_1) = \model(\inn_4) = c_1$ and $\Iclass(head(r_x)) = \{ c_1 \}$
    , and so $r_x$ is coherent with $\model$. 
    Likewise, for $r_y$, we have $\Ifeat(body(r_y)) = \{ \inn_1, \inn_2, \inn_3, \inn_4, \inn_6 \}$, $\model(\inn_1) = \model(\inn_4) = \model(\inn_6) = c_1$, $\model(\inn_2) = \model(\inn_3) = c_2$ and $\Iclass(head(r_y)) = \{ c_1, c_2 \}$, and so $r_y$ is coherent with $\model$. 
    Meanwhile, for $r_z$, we have $\Ifeat(body(r_z)) = \{ \inn_1, \inn_3, \inn_4, \inn_5 \}$, $\model(\inn_1) = \model(\inn_4) = c_1$, $\model(\inn_3) = c_2$ and $\model(\inn_5) = c_3$. 
    However, since $\Iclass(head(r_z)) = \{ c_1 \}$ and thus $\model(\inn_3), \model(\inn_5) \nin \Iclass(head(r_z))$, we can see that $r_z$ is not coherent with $\model$. 

%
%
%
%
\end{example}

Our final example shows how we envisage feedback to operate.

\begin{example}
    Consider an explanation knowledge base for $\model$, $\KB = \Data \cup \Exps$, where $\Data = \{ r_1, r_2, r_3, r_4, r_5, r_6 \}$ and $\Exps = \{ r_x, r_y \}$.
    The following are examples of different feedback rules which may be provided by users:
    \begin{align}
        r_p \!:& \quad (\feat_1, 0) \wedge (\feat_2, 0) \wedge (\feat_3, 1) \Rightarrow c_3    \nonumber \\
        r_q \!:& \quad (f_1, 1) \wedge (f_2, 1) \wedge (f_3, 1) \Rightarrow c_2   \nonumber \\
        r_r \!:& \quad  (f_1, 1) \wedge (f_2, 1) \Rightarrow \neg{c_1}   \nonumber
    \end{align}

The first feedback example, $r_p$, concerns a user specifying how a data point should be classified by providing an instance rule. 
In this case, the data point corresponding to the instance rule, i.e. $\Ifeat(body(r_p)) = \{ \inn_p \}$, \AR{where $\inn_p^1 = 0$, $\inn_p^2 = 0$ and $\inn_p^3 = 1$,} is such that $\model(\inn_p)$ was undefined, and so there is no risk of loss of consistency with the (complete) set of instance rules, i.e. $\Data \cup \{ r_p \}$ is consistent, and so the feedback could be added to $\Data$ (since it is an instance rule) without issue. However, if this were the case, the set of explanation rules $\Exps$ contains $r_y$, and so $\Exps$ would no longer be coherent with $\model$ (by Theorem \ref{thm:main}) since $\Ifeat(body(r_p)) \cap \Ifeat(body(r_y)) \neq \emptyset$ but $\Iclass(head(r_p)) \cap \Iclass(head(r_y)) = \emptyset$. In this case we would need to either modify or remove (parts of) $r_p$ or $\Exps$ to maintain coherence with $\model$, (a strict version of) the second desideratum, while minimising the changes therein as per the third desideratum.

Instance rules may also be provided as feedback rules where they concern a data point for which $\model$ is already defined, i.e. if it is already in a complete $\Data$, such as rule $r_q$.
Here, a conflict exists not only in $\Exps$, i.e. $\{ r_q , r_x \}$ and thus $\{ r_q \} \cup \Exps$ are not consistent, but also with $\Data$, i.e. $\{ r_q , r_4 \}$ and thus $\{ r_q \} \cup \Data$ are not consistent.
Modifications to both parts of $\KB$, or to $r_q$ itself, would thus be needed to satisfy (strict versions of) the desiderata. 

Users may also provide feedback in the form of a generalised rule, possibly containing a negative literal in its head, such as $r_r$. 
Once again, the explanation knowledge base may need to be corrected across both $\Data$ and $\Exps$, e.g. it can be seen here that $\{ r_r , r_1\}$, $\{ r_r , r_4\}$ and $\{ r_r , r_x\}$ are not consistent. 
Note that once $r_q$ has been incorporated, it may be the case that some of these inconsistent sets may already be corrected for, e.g. $\{ r_r , r_4\}$.

\end{example}

\end{document}